\documentclass[envcountsame]{llncs}

\usepackage{latexsym,amssymb,amsmath,amssymb,thmtools,mathtools}
\usepackage{xspace}
\usepackage{graphicx}
\usepackage{misc}
\usepackage{mathabx}
\usepackage{color} 
\usepackage{mathrsfs}
\usepackage{tikz}
\usetikzlibrary{graphs,matrix}
\usepackage[obeyDraft]{todonotes}
\usepackage{standalone}
\usepackage{stackrel}
\usepackage{relsize}

\newcommand{\nb}[1]{\todo[color=red]{#1}\xspace}

\newcommand{\greif}[1]{\ensuremath{\bigocirc#1}}
\newcommand{\lreif}[1]{\ensuremath{\bigodot#1}}

\newcommand\restr[2]{{
  \left.\kern-\nulldelimiterspace 
  #1 
  \vphantom{\big|} 
  \right|_{#2} 
  }}
  
\newcommand{\Int}[1]{#1^{\Imc}\xspace}
\renewcommand\dom{\ensuremath{\Delta}\xspace}
\newcommand{\KB}{\ensuremath{\mathcal{KB}}\xspace}
\newcommand{\per}{\mathpunct{\mbox{\bf .}}}
\newcommand{\pth}[2]{\ensuremath{\textsc{path}_{\mathscr{T}}(#1,#2)}\xspace}
\newcommand{\chd}[2]{\ensuremath{\textsc{child}_{\mathscr{T}}(#1,#2)}\xspace}
\newcommand{\A}{\ensuremath{\mathcal{A}}\xspace}
\newcommand{\Ob}{\ensuremath{\mathcal{O}}\xspace}
\newcommand{\Texa}{\ensuremath{\Tmc_{\sf exa}}\xspace}

\newcommand{\pushright}[1]{\ifmeasuring@#1\else\omit\hfill$\displaystyle#1$\fi\ignorespaces}
\newcommand{\pushleft}[1]{\ifmeasuring@#1\else\omit$\displaystyle#1$\hfill\fi\ignorespaces}
\newcommand{\specialcell}[1]{\ifmeasuring@#1\else\omit$\displaystyle#1$\ignorespaces\fi}

\newcommand{\DL}{\textsl{DL-Lite}}

\allowdisplaybreaks

\title{A Decidable Very Expressive \\ Description Logic for Databases\\ (Extended Version)}

\author{Alessandro Artale, Enrico Franconi, Rafael Pe\~naloza, Francesco Sportelli}
\institute{KRDB Research Centre, 
Free University of Bozen-Bolzano, Italy\\
\texttt{\{artale,franconi,penaloza,sportelli\}@inf.unibz.it}
}

\begin{document}

\date{}
\maketitle

\begin{abstract}
 We introduce \DLRp{}\negmedspace, an extension of the $n$-ary propositionally closed description logic \DLR to deal with attribute-labelled tuples (generalising the positional notation), projections of relations, and global and local objectification of relations, able to express inclusion, functional, key, and external uniqueness dependencies. The logic is equipped with both TBox and ABox axioms. We show how a simple syntactic restriction on the appearance of projections sharing common attributes in a \DLRp knowledge base makes reasoning in the language decidable with the same computational complexity as \DLR. The obtained \DLRpm $n$-ary description logic is able to encode more thoroughly conceptual data models such as EER, UML, and ORM.
\end{abstract}

\section{Introduction}

We introduce the description logic (DL) \DLRp extending the $n$-ary DL
\DLR \cite{Calvanese:et:al:TOCL-2008}, in order to capture database oriented constraints.
%
%
While \DLR is a rather expressive logic, it lacks a number of expressive means that can be added without increasing the complexity of reasoning---when used in a carefully controlled way.
The added expressivity is motivated by the increasing use of DLs as an abstract conceptual layer (an \textit{ontology}) over relational databases.


A \DLR knowledge base can express axioms with (i) propositional
combinations of concepts and (compatible) $n$-ary relations, (ii)
concepts as unary projections of $n$-ary relations, and (iii)
relations with a selected typed component. For
example, if ${\tt Pilot}$ and ${\tt RacingCar}$ are concepts and
${\tt DrivesCar}$, ${\tt DrivesMotorbike}$, ${\tt DrivesVehicle}$ are
binary relations, the knowledge base:
%
\begin{align*}
{\tt Pilot} &\sqsubseteq \exists[1] \selects{2}{\tt RacingCar}{\tt DrivesCar}\\
{\tt DrivesCar}\sqcup {\tt DrivesMotorbike} &\sqsubseteq {\tt DrivesVehicle}
\end{align*}
asserts that a pilot drives a racing car and that driving a car or a motorbike implies driving a vehicle.

The language we propose here, \DLRp, extends \DLR in the following ways.

\begin{itemize}
\item While \DLR instances of $n$-ary relations are $n$-tuples of objects---whose components are identified by their position in the tuple---instances of relations in \DLRp are \emph{attribute-labelled tuples} of objects, i.e., tuples where each component is identified by an attribute and not by its position in the tuple (see, e.g.,~\cite{Kanellakis:1990:chapter}). 
For example, the relation ${\tt Employee}$ may have the signature: 
$${\tt Employee}({\tt firstname}, {\tt lastname}, {\tt dept}, {\tt deptAddr}),$$ 
and an instance of ${\tt Employee}$ could be the tuple:
$$\langle{\tt firstname:John, lastname:Doe, dept: Purchase, deptAddr: London}\rangle.$$
\item Attributes can be \emph{renamed}, for example to recover the positional attributes:
$${\tt firstname, lastname, dept, deptAddr}~\rightleftarrows ~1,2,3,4.$$
\item \emph{Relation projections} allow to form new relations by projecting
  a given relation on some of its attributes. For example, if ${\tt Person}$ is a relation with signature ${\tt Person}({\tt name},{\tt surname})$, it could be related to ${\tt Employee}$ as follows::
\begin{align*}
\pi[{\tt firstname, lastname}]{\tt Employee}\sqsubseteq {\tt Person},\\
{\tt firstname, lastname}~\rightleftarrows ~\tt{name,surname}.
 \end{align*}
\item The \emph{objectification} of a relation (also known as \emph{reification}) is a concept whose instances are unique identifiers of the tuples instantiating the relation. Those identifiers could be unique only within an objectified relation (\emph{local objectification}), or they could be uniquely identifying tuples independently on the relation they are instance of (\emph{global objectification}). For example, the concept ${\tt EmployeeC}$ could be the \textit{global} objectification of the relation ${\tt Employee}$, assuming that there is a global 1-to-1 correspondence between pairs of values of the attributes ${\tt firstname,lastname}$ and ${\tt EmployeeC}$ instances:
$${\tt EmployeeC}\equiv\greif{\exists[{\tt firstname,lastname}] {\tt Employee}}.$$

Consider the relations with the following signatures:
$${\tt DrivesCar(name, surname, car)}, \quad {\tt OwnsCar (name, surname, car)},$$
and assume that anybody driving a car also owns it: ${\tt DrivesCar \sqsubseteq OwnsCar}$. The \textit{locally} objectified events of driving and owning, defined as 
$${\tt CarDrivingEvent \equiv \lreif DrivesCar},\quad {\tt CarOwningEvent \equiv \lreif OwnsCar},$$
do not imply that a driving event by a person of a car is the owning event by the same person and the same car: ${\tt CarDrivingEvent \not\sqsubseteq CarOwningEvent}$. Indeed, they are even disjoint: ${\tt CarDrivingEvent \sqcap CarOwningEvent \sqsubseteq \bot}$.

\end{itemize}


It turns out that \DLRp is an expressive description logic able to assert relevant constraints typical of relational databases. In Section~\ref{sec:expressivity} we will consider \emph{inclusion dependencies}, \emph{functional and key dependencies}, \emph{external uniqueness} and \emph{identification} axioms.
For example, \DLRp can express the fact that the attributes ${\tt firstname,lastname}$ play the role of a multi-attribute key for the relation ${\tt Employee}$: 
$$\pi[{\tt firstname,lastname}] {\tt Employee}\sqsubseteq
\pi^{\leq 1}[{\tt firstname,lastname}]{\tt Employee},$$
and that the attribute ${\tt deptAddr}$ functionally depends on the attribute ${\tt dept}$ within the relation ${\tt Employee}$:
$$
\exists[{\tt dept}] {\tt Employee}\sqsubseteq
\exists^{\leq 1}[{\tt dept}]\left(\pi[{\tt dept,deptAddr}]{\tt Employee}\right).
$$

While \DLRp turns out to be undecidable, we show how a simple syntactic condition on the appearance of projections sharing common attributes in a knowledge base makes the language decidable. The result of this restriction is a new language called \DLRpm{}\negmedspace. We prove that \DLRpm{}\negmedspace, while preserving most of the \DLRp expressivity, has a reasoning problem whose complexity does not increase w.r.t. the computational complexity of the basic \DLR language. We also present in Section~\ref{sec:api} the implementation of an API for the reasoning services in \DLRpm.

\begin{figure}[t]
	\centering
		\renewcommand{\arraystretch}{1.2} 
		$
		\begin{array}{r@{\hspace{2ex}}c@{\hspace{2ex}}l} 
			C & \to & C\!N\ \mid\ \neg C\
                                  \mid\ C_{1}\sqcap C_{2}\ \mid\
                                  \exists^{\geq q}[U_i] R\ \mid\ \greif{R}\ \mid\ \lreif{R\!N}\\
			R & \to & R\!N\ \mid\ R_1\setminus R_2\ \mid\
                                  R_{1}\sqcap R_{2}\mid\ R_{1}\sqcup
                                  R_{2}\mid\ \selects{U_i}{C}{R}\
                                  \mid\ \pi^{\lesseqgtr q}[U_1,\ldots,U_k] R\\
			\varphi & \to & C_1\sqsubseteq C_2\ \mid\ R_1\sqsubseteq R_2 \mid C\!N(o) \mid R\!N(U_1\!:\!o_1,\ldots,U_n\!:\!o_n) \mid o_1 = o_2 \mid o_1 \neq o_2 \\
			\vartheta & \to & U_1 \rightleftarrows U_2
		\end{array}
		$ 
		\renewcommand{\arraystretch}{1} 
	\caption{\label{fig:dlrp} The syntax of \DLRp.} 
\end{figure}
\section{The Description Logic \DLRp}
\label{sec:syntax}

We start by introducing the syntax of \DLRp. A \DLRp \emph{signature}
is a tuple
$\mathcal{L}=(\mathcal{C},\mathcal{R},\mathcal{O},\mathcal{U},\tau)$
where $\mathcal{C}$, $\mathcal{R}$, $\mathcal{O}$ and $\mathcal{U}$
are finite, mutually disjoint sets of \emph{concept names},
\emph{relation names}, \emph{individual names}, and \emph{attributes},
respectively, and $\tau$ is a \emph{relation signature} function,
associating a set of attributes to each relation name
$\tau(R\!N)=\{U_1,\ldots,U_n\}\subseteq \Umc$, with $n\geq 2$.

The syntax of concepts $C$, relations $R$, formulas $\varphi$, and
attribute renaming axioms $\vartheta$ is given in
Figure~\ref{fig:dlrp}, where $C\!N\in\mathcal{C}$,
$R\!N\in\mathcal{R}$, $U\in\mathcal{U}$, $o\in \Ob$, $q$ is a positive
integer and $2\leq k < \textsc{arity}(R)$.  The \emph{arity} of a
relation $R$ is the number of the attributes in its signature; i.e.,
$\textsc{arity}(R)=\left|\tau(R)\right|$, with the relation signature function $\tau$ extended to complex relations as in
Figure~\ref{fig:syn:tau}. Note that it is possible that the same
attribute appears in the signature of different relations. 

\begin{figure}[t] 
	\centering
		\renewcommand{\arraystretch}{1.2} 
		${ 
		\begin{array}{r@{\hspace{1ex}}l@{\hspace{3ex}}l@{\hspace{.3ex}}} 
			\tau(R_1\setminus R_2) = & \tau(R_1) & \\
			\tau(R_{1}\sqcap R_{2}) = & \tau(R_1) &
                                                                \text{if } \tau(R_1) = \tau(R_2)\\
			\tau(R_{1}\sqcup R_{2}) = & \tau(R_1) & \text{if } \tau(R_1) = \tau(R_2)\\
			\tau(\selects{U_i}{C}{R}) = & \tau(R) &
                                                                \text{if } U_i\in \tau(R)\\
			\tau(\pi^{\lesseqgtr q}[U_1,\ldots,U_k] R) = &
                                                               \{U_1,\ldots,U_k\} & \text{if } \{U_1,\ldots,U_k\}\subset \tau(R)\\
			\text{undefined} & & \text{otherwise} 
		\end{array}
		}$ 
		\renewcommand{\arraystretch}{1}
	\caption{\label{fig:syn:tau} The signature of \DLRp relations.} 
\end{figure}

As mentioned in the introduction, the \DLRp constructors added to \DLR are the \emph{local} and \emph{global
  objectification} ($\lreif{R\!N}$ and $\greif{R}$, respectively);
\emph{relation projections} with the possibility to count the
projected tuples ($\pi^{\lesseqgtr q}[U_1,\ldots,U_k] R$), and \emph{renaming
  axioms} over attributes ($U_1 \rightleftarrows
U_2$). 
Note that local objectification
($\lreif{R}$) can be applied to relation names, while global objectification
($\greif{R\!N}$) can be applied to complex relations. We use the
standard abbreviations:
\begin{align*}
  &\bot = C\sqcap\neg C,~~\top = \neg\bot,~~
  C_1 \sqcup C_2 = \neg(\neg C_1 \sqcap \neg C_2),~~ \exists[U_i] R = \exists^{\geq 1}[U_i]R,\\
  &\exists^{\leq q}[U_i]R  = \neg(\exists^{\geq q + 1}[U_i] R),~~
  \pi[U_1,\ldots,U_k] R = \pi^{\geq 1}[U_1,\ldots,U_k]R.
\end{align*}

A \DLRp \emph{TBox} \Tmc is a finite set of \emph{concept inclusion}
axioms of the form $C_1\sqsubseteq C_2$ and \emph{relation inclusion}
axioms of the form $R_1\sqsubseteq R_2$. We use $X_1\equiv X_2$ as a
shortcut for the two axioms $X_1\sqsubseteq X_2$ and
$X_2\sqsubseteq X_1$. A \DLRp \emph{ABox} \Amc is a finite set of
\emph{concept instance} axioms of the form $C\!N(o)$, \emph{relation
  instance} axioms of the form $R\!N(U_1\!:\!o_1,\ldots,U_n\!:\!o_n)$,
and \emph{same/distinct individual} axioms of the form $o_1 = o_2$ and
$o_1 \neq o_2$, with $o_i\in\Ob$. Restricting ABox axioms to concept
and relation names only does not affect the expressivity of \DLRp due
to the availability of unrestricted TBox axioms. 
A \DLRp \emph{renaming schema} $\Re$ is a finite set of renaming axioms of the form $U_1 \rightleftarrows U_2$. We use the
shortcut $U_1\ldots U_n\rightleftarrows U'_1\ldots U'_n$ to group many
renaming axioms with the meaning that $U_i\rightleftarrows U'_i$ for
all $i=1,\ldots, n$.
A \DLRp knowledge base (KB)
$\mathcal{KB}=(\Tmc\!,\Amc,\Re)$ is composed by a TBox \Tmc, an ABox
\Amc, and a renaming schema $\Re$.

The renaming operator $\rightleftarrows$ is an equivalence relation over the attributes $\mathcal{U}$, $(\rightleftarrows,\mathcal{U})$. The partitioning of $\mathcal{U}$ into equivalence classes induced by a renaming schema is meant to represent the alternative ways to name attributes in the knowledge base. A unique \emph{canonical representative} for each equivalence class is chosen to replace all the attributes in the class throughout the knowledge base. From now on we assume that a knowledge base is consistently rewritten by substituting each attribute with its canonical representative. After this rewriting, the renaming schema does not play any role in the knowledge base. We allow only \emph{arity-preserving} renaming schemas, i.e., there is no equivalence class containing two attributes from the same relation signature.


%

As shown in the introduction, the renaming schema is useful to reconcile the named attribute perspective and the positional perspective on relations.
It is also important to enforce \emph{union compatibility} among relations involved in relation inclusion axioms, and among relations involved in $\sqcap$- and $\sqcup$-set expressions.
%
%
Two relations are  \emph{union compatible} (w.r.t. a
  renaming schema) if they have the same
signature (up to the attribute renaming induced by the
renaming schema).  Indeed, as it will be clear from the semantics, a relation inclusion axiom involving non union compatible relations would always be false, and a $\sqcap$- and $\sqcup$-set expression involving non union compatible relations would always be empty.

%
%


\vspace{2ex}

\begin{figure}[t]
  \centering
        \renewcommand*{\arraystretch}{1.2}
		$ 
		\begin{array}{@{}r@{\hspace{0.5ex}}l@{}} 
			%
			%
			\Int{(\neg C)} = & \Int{\top} \setminus \Int{C}\\
			\Int{(C_{1}\sqcap C_{2})} = & \Int C_{1} \cap \Int C_{2}\\
			%
			%
			\Int{(\exists^{\geq q}[U_i] R)} = & \{d\in\dom\mid~\left|\{t\in\Int R\mid t[U_i]=d\}\right| \geq q \}\\
			\Int{(\greif{R})} = & \{d\in \dom \mid d=\imath(t) \land t\in \Int{R}\}\\
			\Int{(\lreif{R\!N})} = & \{d\in \dom \mid d=\ell_{R\!N}(t)\land t\in \Int{R\!N}\}\\
			\Int{(R_1\setminus R_2)} = & \Int R_{1} \setminus \Int R_{2}\\
			\Int{(R_{1}\sqcap R_{2})} = & \Int R_{1} \cap \Int R_{2}\\
			\Int{(R_{1}\sqcup R_{2})} = & \{t\in\Int R_{1}\cup\Int R_{2}\mid \tau(R_1)= \tau(R_2)\}\\
			\Int{(\selects {U_i}{C}{R})} = & \{
			t\in\Int{R} \mid t[U_i]\in\Int{C}\} \\
			\Int{(\pi^{\lesseqgtr q}[U_1,\ldots,U_k] R)} =
                                     &\{\langle U_1:d_1,\ldots,U_k:d_k\rangle\in T_{\dom}(\{U_1,\ldots,U_k\})\mid%
			\hspace{-1.5ex}%
			\\%
			& ~1\leq\left|\{t\in\Int R\mid t[U_1]=d_1,\ldots,t[U_k]=d_k\}\right| \lesseqgtr q \}
		\end{array}
		$
        \renewcommand*{\arraystretch}{1}
	\caption{The semantics of \DLRp expressions.} 
	\label{fig:sem:dlrp}
\end{figure}

The semantics of \DLRp uses the notion of \emph{labelled tuples} over a potentially infinite domain 
$\Delta$. Given a set of labels $\mathcal{X} \subseteq \mathcal{U}$ an \emph{$\mathcal{X}$-labelled tuple over $\Delta$} (or \emph{tuple} for short) is a \emph{total} function $t \colon \mathcal{X} \to \Delta$. For $U\in \mathcal{X}$, we write $t[U]$ to refer to the domain element ${d\in \Delta}$ labelled by $U$. 
Given $d_1,\dots,d_n\in \Delta$, the expression ${\langle U_1\colon d_1,\ldots,U_n\colon d_n\rangle}$ stands for the tuple $t$ defined on the set of labels $\{U_1,\ldots,U_n\}$ such that ${t[U_i]=d_i}$, for ${1\leq 1\leq n}$.
The \emph{projection} of the tuple $t$ over the attributes ${U_1,\ldots,U_k}$ is the function $t$ restricted to be undefined for the labels not in ${U_1,\ldots,U_k}$, and it is denoted by ${t[U_1,\ldots,U_k]}$. The relation signature function $\tau$ is extended to labelled tuples to obtain the set of labels on which a tuple is defined. $T_\Delta(\mathcal{X})$ denotes the set of all $\mathcal{X}$\mbox{-}labelled tuples over $\Delta$, for $\mathcal{X} \subseteq \mathcal{U}$, and we overload this notation by denoting with $T_\Delta(\mathcal{U})$ the set of all possible tuples with labels within the whole set of attributes $\Umc$.

A \DLRp \emph{interpretation} is a tuple
$\Imc = (\dom, \cdot^\Imc, \imath, L)$ consisting of a nonempty
\emph{domain} $\dom$, an \emph{interpretation function} $\cdot^\Imc$, 
a \emph{global objectification function} $\imath$, and a family $L$
containing one \emph{local objectification function} $\ell_{R\!N_i}$
for each named relation $R\!N_i\in\mathcal{R}$.
%
%
The global objectification function is an injective function,
${\imath:T_{\dom}(\Umc) \to \dom}$, associating a \emph{unique} global
identifier to each tuple.
%
The local objectification functions,
${\ell_{R\!N_i}:T_{\dom}(\Umc) \to \dom}$, are associated to each
relation name in the signature, and as the global objectification
function they are injective: they associate an identifier---which is
guaranteed to be unique only within the interpretation of a relation
name---to each tuple.

The interpretation function $\cdot^\Imc$ assigns a domain element to
each individual, $o^{\mathcal{I}}\in\dom$, a set of domain elements to
each concept name, $C\!N^{\mathcal{I}}\subseteq \dom$, and a set of
$\tau(R\!N)$-labelled tuples over $\dom$ to each relation name $R\!N$, $R\!N^{\mathcal{I}}\subseteq T_{\dom}(\tau(R\!N))$. 
Note that the unique name assumption is not enforced.
The interpretation function $\cdot^\Imc$ is unambiguously extended
over concept and relation expressions as specified in
Figure~\ref{fig:sem:dlrp}. Notice that the construct
$\pi^{\lesseqgtr q}[U_1,\ldots,U_k] R$ is interpreted as a classical
projection over a relation, thus including only tuples belonging to
the relation.

The interpretation $\Imc$ satisfies the concept inclusion axiom
$C_1\sqsubseteq C_2$ if $\Int C_1\subseteq \Int C_2$, and the relation
inclusion axiom $R_1\sqsubseteq R_2$ if $\Int R_1\subseteq \Int
R_2$. It satisfies the concept instance axiom $C\!N(o)$ if
$\Int o\in \Int {C\!N}$\!, the relation instance axiom
$R\!N(U_1\!:\!o_1,\ldots,U_n\!:\!o_n)$ if
${\langle U_1\colon\Int{o_1},\ldots,U_n\colon\Int{o_n}\rangle}\in
\Int {R\!N}$, and the axioms $o_1 = o_2$ and $o_1 \neq o_2$ if
$\Int{o_1} = \Int{o_2}$, and $\Int{o_1} \neq \Int{o_2}$,
respectively. 
\Imc is a \emph{model} of the knowledge base $(\mathcal{T}\!,\mathcal{A},\Re)$ if
it satisfies all the axioms in the TBox $\mathcal{T}$ and in the ABox
$\mathcal{A}$, once the knowledge base has been rewritten according to the
renaming schema.


\begin{example}\label{exa:basic}
  Consider the relation names $R_1, R_2$ with
  $\tau(R_1) = \{W_1,W_2,W_3,W_4\}$,
  $\tau(R_2) = \{V_1,V_2,V_3,V_4,V_5\}$, and a knowledge base with the renaming axiom
  $W_1W_2W_3 \rightleftarrows V_3V_4V_5$ and a TBox \Texa:
 \begin{align}
   \pi[W_1,W_2] R_1 \sqsubseteq {} & \pi^{\leq 1}[W_1,W_2] R_1 \label{key} \\
   \pi[V_3,V_4] R_2 \sqsubseteq {} & \pi^{\leq 1}[V_3,V_4](\pi[V_3,V_4,V_5] R_2) \label{funct-dep} \\
   \pi[W_1,W_2,W_3] R_1 \sqsubseteq {} & \pi[V_3,V_4,V_5] R_2 \label{subrelation}.
 \end{align}
The axiom~\eqref{key}
 expresses that $W_1,W_2$ form a multi-attribute key for $R_1$;
 \eqref{funct-dep} introduces a functional dependency in the relation
 $R_2$ where the attribute $V_5$ is functionally dependent from
 attributes $V_3, V_4$, and~\eqref{subrelation} states an inclusion
 between two projections of the relation names $R_1, R_2$ based on
 the renaming schema axiom.\qed
\end{example}

\emph{KB satisfiability} refers to the problem of deciding the existence of a model of a given knowledge base; \emph{concept satisfiability} (resp. \emph{relation satisfiability}) is the problem of deciding whether there is a model of the knowledge base with a non-empty interpretation of a given concept (resp. relation). A knowledge base \emph{entails} (or \emph{logically implies}) an axiom if all models of the knowledge base are also models of the axiom.
For instance, it is easy to see that the TBox in Example~\ref{exa:basic} entails that $V_3,V_4$ are a key for $R_2$:
$$
 \Texa\models \pi[V_3,V_4] R_2 \sqsubseteq \pi^{\leq 1}[V_3,V_4] R_2 ~,
$$

\noindent and that axiom~\eqref{funct-dep} is redundant in \Texa.
The decision problems in \DLRp can be all reduced to KB satisfiability.
\begin{lemma}
In \DLRp, concept and relation satisfiability and entailment are reducible to KB satisfiability.
\end{lemma}

\section{Expressiveness of \DLRp}
\label{sec:expressivity}

\DLRp is an expressive description logic able to assert relevant constraints in the context of relational databases, such as \emph{inclusion dependencies} (namely inclusion axioms among arbitrary
projections of relations), \emph{equijoins}, \emph{functional dependency} axioms, \emph{key} and \emph{foreign key} axioms, \emph{external uniqueness} axioms,  \emph{identification} axioms, and \emph{path functional dependencies}.

An \emph{equijoin} among two relations with disjoint signatures is the set of all combinations of tuples in the relations that are equal on their selected attribute names. Let $R_1,R_2$ be relations with signatures
$\tau(R_1)=\{U,U_1,\ldots,U_{n_1}\}$ and $\tau(R_2)=\{V,V_1,\ldots,V_{n_2}\}$; their equijoin over $U$ 
and $V$ is the relation  $R = R_1\!\stackrel[U=V]{}{\mathlarger{\mathlarger{\mathlarger{\bowtie}}}}\!R_2$ with 
signature $\tau(R)=\tau(R_1)\cup\tau(R_2)\setminus\{V\}$, 
which is expressed by the \DLRp axioms:
\renewcommand*{\arraystretch}{1.3}
$$\begin{array}{l}
\pi[U,U_1,\ldots,U_{n_1}]R  \equiv {} \selects{U}{(\exists[U]R_1\sqcap \exists[V]R_2)}{R_1} \\
\pi[V,V_1,\ldots,V_{n_2}]R \equiv {} \selects{V}{(\exists[U]R_1\sqcap \exists[V]R_2)}{R_2}
\\  U\rightleftarrows V~.
\end{array}
\renewcommand*{\arraystretch}{1}
$$
A \emph{functional dependency} axiom $(R:U_1\ldots U_j \rightarrow U)$ 
(also called \emph{internal uniqueness} axiom~\cite{halpin2008}) states that the values of the attributes $U_1\ldots U_j$ uniquely determine the value of the attribute $U$ in the relation $R$. 
Formally, the interpretation $\Imc$ satisfies this functional dependency axiom 
if, for all tuples $s,t\in R^\Imc$, $s[U_1] = t[U_1], \ldots, s[U_j] = t[U_j]$ imply $s[U] = t[U]$. 
Functional dependencies can be expressed in 
\DLRp{}\negmedspace, assuming that $\{U_1,\ldots, U_j, U\}\subseteq \tau(R)$, with the axiom:
\begin{center}
  $\pi[U_1,\ldots, U_j]R\sqsubseteq\pi^{\leq 1}[U_1,\ldots,
  U_j](\pi[U_1,\ldots, U_j, U] R).$
\end{center}
A special case of a functional dependency are \emph{key} axioms $(R:U_1\ldots U_j \rightarrow R)$, which state
that the values of the key attributes $U_1\ldots U_j$ of a relation $R$ uniquely identify tuples in $R$. A key axiom can be expressed in \DLRp, assuming that $\{U_1\ldots U_j\}\subseteq\tau(R)$, with the axiom:
\begin{center}
  $\pi[U_1,\ldots, U_j]R\sqsubseteq\pi^{\leq 1}[U_1,\ldots,
  U_j]R.$
\end{center}

A \emph{foreign key}\nb{A: added} is the obvious result of an inclusion dependency
together with a key constraint involving the foreign key attributes.

The \emph{external uniqueness} axiom $([U^1]R_1\downarrow\ldots \downarrow [U^h]R_h)$ states that the join $R$ of the relations $R_1,\ldots,R_h$ via the attributes $U^1,\ldots,U^h$ has the joined attribute functionally dependent on all the others~\cite{halpin2008}. This can be expressed in \DLRp with the axioms:
%
\renewcommand*{\arraystretch}{1.3}
$$\begin{array}{l}
R\equiv {} R_1\!\stackrel[U^1=U^2]{}{\mathlarger{\mathlarger{\mathlarger{\bowtie}}}}\cdots\stackrel[U^{h-1}=U^{h}]{}{\mathlarger{\mathlarger{\mathlarger{\bowtie}}}}\!R_h\\
R: {}  U^1_{1},\ldots, U^1_{n_1},\ldots,U^h_{1},\ldots, U^h_{n_h} \rightarrow U^1
\end{array}
\renewcommand*{\arraystretch}{1}
$$
where $\tau(R_i)=\{U^i,U^i_1,\ldots, U^i_{n_i}\}, 1\le i\le h$, and $R$ is a new relation name with 
%
$\tau(R)=\{U^1,U^1_{1},\ldots, U^1_{n_1},\ldots,U^h_{1},\ldots, U^h_{n_h}\}$.

\emph{Identification} axioms as defined in \DLRID~\cite{CalvaneseGL01}
(an extension of \DLR with functional dependencies and identification
axioms) are a variant of external uniqueness axioms, constraining only
the elements of a concept $C$; they can be expressed in \DLRp with the
axiom:
\begin{center}
  $ [U^1]\selects{U_1}{C}{R_1}\downarrow\ldots \downarrow
  [U^h]\selects{U_h}{C}{R_h}.$
\end{center}


\emph{Path functional dependencies}---as defined in the DL family
$\mathcal{CFD}$~\cite{TomanW09}---can be
expressed in \DLRp as identification axioms involving joined sequences
of functional binary relations.  \DLRp also captures the \emph{tree-based identification
  constraints (tid)} introduced in~\cite{CFPSS-AAAI-2014-dlnf} to express functional dependencies in $\DL_{\textit{RDFS,tid}}$. 
  The rich set of constructors in \DLRp allows us to extend the known mappings in description logics of popular conceptual data models. The EER mapping as introduced in~\cite{ACKRZ:er07} can be extended to deal with multi-attribute keys (by using identification axioms) and named roles in relations; the ORM mapping as introduced in~\cite{DBLP:conf/otm/FranconiMS12,DBLP:conf/otm/SportelliF16} can be extended to deal with arbitrary subset and exclusive relation constructs (by using inclusions among global objectifications of projections of relations), arbitrary internal and external uniqueness constraints, arbitrary frequency constraints (by using projections), local objectification, named roles in relations, and fact type readings (by using renaming axioms); the UML mapping as introduced in~\cite{BeCD05-AIJ-2005} can be fixed to deal properly with association classes (by using local objectification) and named roles in associations.

\section{The \DLRpm fragment of \DLRp}
\label{sec:dlrpm}

Since a \DLRp knowledge base can express inclusions and functional dependencies, the entailment problem is undecidable~\cite{ChandraV85}. Thus, in this section we present \DLRpm{}\negmedspace, a decidable syntactic fragment of \DLRp limiting the coexistence of relation projections in a knowledge base.
 
Given a \DLRp knowledge base $\KB = (\mathcal{T}\!,\mathcal{A},\Re)$, we
define the \emph{projection signature of} \KB as the set $\mathscr{T}$
containing the signatures $\tau(R\!N)$ of all relations
$R\!N\in\mathcal{R}$, the singleton sets associated with each
attribute name $U\in\mathcal{U}$, and the relation signatures that
appear explicitly in projection constructs in some axiom from \Tmc,
together with their implicit occurrences due to the renaming
schema. Formally, $\mathscr{T}$ is the smallest set such that
%
	(i)~$\tau(R\!N)\in\mathscr{T}$ for all $R\!N\in\mathcal{R}$; 
	(ii) $\{U\}\in\mathscr{T}$ for all $U\in\mathcal{U}$; and
	(iii) $\{U_1,\ldots,U_k\}\in\mathscr{T}$ for all $\pi^{\lesseqgtr q}[V_1,\ldots,V_k] R$ appearing as sub-formulas in $\Tmc$ and 
		$V_i\in [U_i]_\Re $ \text{ for}~$1\!\leq\!i\!\leq\!k$. 

The \emph{projection signature graph} of \KB is the directed acyclic graph corresponding to the Hasse diagram of $\mathscr{T}$ ordered by the proper subset relation $\supset$, whose sinks are the attribute singletons $\{U\}$. We call this graph $(\supset,\mathscr{T})$.
%
Given a set of attributes
$\tau=\{U_1,\ldots,U_k\}\subseteq\mathcal{U}$, the \emph{projection
  signature graph dominated by $\tau$}, denoted as $\mathscr{T}_\tau$,
is the sub-graph of $(\supset,\mathscr{T})$ with $\tau$ as root and
containing all the nodes reachable from $\tau$.
%
%
Given two sets of attributes $\tau_1,\tau_2\subseteq\mathcal{U}$,
$\pth{\tau_1}{\tau_2}$ denotes the set of paths in
$(\supset,\mathscr{T})$ between $\tau_1$ and $\tau_2$. Note that,
$\pth{\tau_1}{\tau_2}=\emptyset$ both when a path does not exist and
when $\tau_1\subseteq \tau_2$.  The notation $\chd{\tau_1}{\tau_2}$
means that ${\tau_2}$ is a child (i.e., a direct descendant) of
${\tau_1}$ in $(\supset,\mathscr{T})$.
We now introduce \DLRpm as follows.
\begin{definition}\label{def:dlrpm}
	A \emph{\DLRpm knowledge base} is a \DLRp knowledge base that satisfies the following syntactic conditions: 
	\begin{enumerate}
        \item the projection signature graph $(\supset,\mathscr{T})$
          is a multitree: i.e., for every node $\tau\in\mathscr{T}$,
          the graph $\mathscr{T}_\tau$ is a tree; and
        \item for every projection construct
          $\pi^{\lesseqgtr q}[U_1,\ldots,U_k] R$ and every concept
          expression of the form $\exists^{\geq q}[U] R$ appearing
          in \Tmc, if $q>1$ then the length of the path
          $\pth{\tau(R)}{\{U_1,\ldots,U_k\}}$ is 1.
	\end{enumerate}
\end{definition}
The first condition in \DLRpm restrict \DLRp in the way that multiple
projections of relations may appear in a knowledge base: intuitively, there cannot
be different projections sharing a common attribute.  Moreover,
observe that in \DLRpm $\textsc{path}_{\mathscr{T}}$ is necessarily
functional, due to the multitree restriction. By relaxing the first condition the language becomes undecidable, as we mentioned at the beginning of this Section. The second condition is also necessary to prove decidability of \DLRpm (see the proof in the next Section); however, we do not know whether this condition could be relaxed while preserving decidability.

Figure~\ref{fig:multitree} shows that the projection signature graph of the knowledge base from Example~\ref{exa:basic} is indeed a multitree.
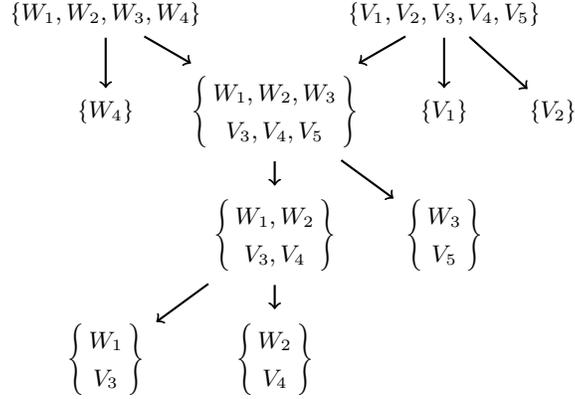
\begin{figure}[t]
\centering


\begin{tikzpicture}[thick,shorten >=1pt, shorten <=1pt,->]
  \matrix (m) [matrix of math nodes, row sep=1.3em, column sep=-1.1em]{
	 \{W_1,W_2,W_3,W_4\} & & \{V_1,V_2,V_3,V_4,V_5\} & \\
	\{W_4\} & \left\{ \begin{array}{c} W_1,W_2,W_3 \\ V_3,V_4,V_5 \end{array} \right\} & \{V_1\} & \{V_2\} \\
	& \left\{\begin{array}{c} W_1,W_2 \\ V_3,V_4 \end{array}\right\} & \left\{ \begin{array}{c} W_3 \\ V_5 \end{array} \right\}\\
	\left\{ \begin{array}{c} W_1 \\ V_3 \end{array} \right\}  & \left\{ \begin{array}{c} W_2 \\ V_4 \end{array} \right\} & \\
  };
  \path 
  	(m-1-1) edge (m-2-1)
  	(m-1-1) edge (m-2-2)
  	(m-1-3) edge (m-2-2)
  	(m-1-3) edge (m-2-3)
  	(m-1-3) edge (m-2-4)
  	(m-2-2) edge (m-3-2)
  	(m-2-2) edge (m-3-3)
  	(m-3-2) edge (m-4-1)
  	(m-3-2) edge (m-4-2)
  ;
\end{tikzpicture}
\caption{\label{fig:multitree} The projection signature graph of Example~\ref{exa:basic}.}
\end{figure}
Note that in the figure we have collapsed equivalent attributes in a unique equivalence class, according to the renaming schema. Furthermore, since all its projection constructs have $q=1$, this knowledge base belongs to \DLRpm.
%

\DLR is included in \DLRpm, since the projection signature graph of any \DLR knowledge base is always a degenerate multitree with maximum depth equal to 1. 
Not all the database constraints as introduced in
Section~\ref{sec:expressivity} can be directly expressed in
\DLRpm. While functional dependency and key axioms can be expressed
directly in \DLRpm, equijoins, external uniqueness axioms, and
identification axioms introduce projections of a relation which share
common attributes, thus violating the multitree restriction. 
For example,\nb{A: Added} the axioms for capturing an equijoin between two
relations, $R_1, R_2$ would generate a projection signature graph with
the signatures of $R_1, R_2$ as projections of the signature of the
join relation $R$ sharing the attribute on which the join is performed,
thus violating condition $1$.

However,
in \DLRpm it is still possible to reason over both external uniqueness
and identification axioms by encoding them into a set of saturated
ABoxes (as originally proposed in~\cite{CalvaneseGL01}) and check
whether there is a saturation that satisfies the
constraints. Therefore, we can conclude that \DLRID extended with
unary functional dependencies is included in \DLRpm{}\negmedspace,
provided that projections of relations in the knowledge base form a
multitree projection signature graph. Since (unary) functional
dependencies are expressed via the inclusions of projections of
relations, by constraining the projection signature graph to be a
multitree, the possibility to build combinations of functional
dependencies as the ones in~\cite{CalvaneseGL01} leading to
undecidability is ruled out. 

Note that the \emph{non-conflicting keys} sufficient condition guaranteeing the decidability of inclusion dependencies and keys of \cite{j.websem289} is in conflict with our more restrictive requirement: indeed \cite{j.websem289} allow for overlapping projections, but the considered datalog language is not comparable to \DLRp.

Concerning the ability of \DLRpm to capture conceptual data models, only the mapping of ORM schemas is affected by the \DLRpm restrictions: \DLRpm is able to correctly express an ORM schema if the projections involved in the schema satisfy the \DLRpm multitree restriction.


\section{Mapping \DLRpm to \ALCQI}
\label{sec:mapping}

This section shows constructively the main technical result
of this paper, i.e., that reasoning in \DLRpm is an \ExpTime-complete
problem. The lower bound is clear by observing that \DLR is a
sublanguage of \DLRpm. More challenging is the upper bound obtained by
providing a mapping from \DLRpm KBs to \ALCQI KBs---a Boolean complete
DL with qualified number restrictions of the form $\exists^{\geq
  q}R\per C$, and inverse roles of the form $R^-$
(see~\cite{BCMNP03} for more details).
%
%
We adapt and extend the mapping presented for \DLR
in~\cite{Calvanese:et:al:TOCL-2008}, with the modifications proposed
by~\cite{HorrocksSTT00} to deal with ABoxes without the unique name
assumption.

We recall that the renaming schema, $\Re$, does not play any role
since we assumed that a \DLRpm KB is rewritten by choosing a single
canonical representative, $[U]_\Re$, for each $V\in [U]_\Re$. Thus, we
consider \DLRpm KBs as pairs of TBox and ABox axioms.

\begin{figure}[t] 
\centering
$ 
\begin{array}{r@{\hspace{1ex}}c@{\hspace{1ex}}l} 
(\neg C)^\dag &=& \neg C^\dag \\
(C_1 \sqcap C_2)^\dag & =& C_1^\dag \sqcap C_2^\dag \\
%
%
(\exists^{\geq q}[U_i] R)^\dag & =& \left\{ 
                                  \begin{array}[2]{ll}
                                          \exists^{\geq q}{\left(\pth{\tau(R)}{\{U_i\}}^\dag\right)^-}\per{R^\dag},
                                           & \text{if}~\pth{\tau(R)}{\{U_i\}}\neq\emptyset\\
                                          \bot, & \text{otherwise}
                                   \end{array}\right.
                                                            \\
(\greif{R})^\dag &=& R^\dag \\
(\lreif{R\!N})^\dag &=& A^l_{R\!N}
\vspace{2ex}\\
(R_1\setminus R_2)^\dag &=& R_1^\dag\sqcap \neg R_2^\dag\\
(R_1 \sqcap R_2)^\dag & =& R_1^\dag \sqcap R_2^\dag\\
(R_1 \sqcup R_2)^\dag & =& \left\{
                                \begin{array}[2]{ll}
                                   R_1^\dag \sqcup R_2^\dag, & \text{if}~\tau(R_1)=\tau(R_2)\\
                                  \bot, & \text{otherwise}
                                \end{array}\right.\\
(\selects{U_i}{C}{R})^\dag & =& \left\{
                                    \begin{array}[2]{ll} 
                                        R^\dag \sqcap  \forall \pth{\tau(R)}{\{U_i\}}^\dag\per C^\dag,
                                           & \text{if}~\pth{\tau(R)}{\{U_i\}}\neq\emptyset\\
                                       \bot, & \text{otherwise}
                                     \end{array}\right.\\
(\pi^{\lesseqgtr q}[U_1,\ldots,U_k]R)^\dag & =& \left\{
                                      \begin{array}[1]{l}
                                         \cardd{q}{\left(\pth{\tau(R)}{\{U_1,\ldots,U_k\}}^\dag\right)^-}{R^\dag},\\
                                          \qquad\qquad\qquad\qquad\text{if}~\pth{\tau(R)}{\{U_1,\ldots,U_k\}}\neq\emptyset\\
                                          \bot, \qquad\qquad\qquad\qquad\text{otherwise}
                                        \end{array}\right.
\end{array}
$ 
\caption{The mapping to \ALCQI{} for concept and relation expressions.} \label{fig:themapping}
\end{figure}

We first introduce a mapping function $\cdot^\dag$ from \DLRpm
concepts and relations to \ALCQI concepts.
The function $\cdot^\dag$ maps each concept name $C\!N$ and each
relation name $R\!N$ appearing in the \DLRpm KB to an \ALCQI concept
names $C\!N$ and $A_{R\!N}$, respectively. The latter is the global
reification of $R\!N$.
For each relation name $R\!N$, the \ALCQI signature also includes a
concept name $A_{R\!N}^{l}$ and a role name $Q_{R\!N}$ to capture
local objectification.  The mapping $\cdot^\dag$ is extended to
concept and relation expressions as illustrated in Figure~\ref{fig:themapping},
where\nb{A: added} the notation $\cardd{q}{R}{C}$ is a shortcut for the conjunction
$\exists R\per C\sqcap \exists^{\geq q}{R}\per{C}$.
%

The mapping crucially uses the projection signature graph to
map projections and selections, by accessing paths in the projection
signature graph $(\supset,\mathscr{T})$ associated to the \DLRpm
KB. If there is a path
$\pth{\tau}{\tau'} = \tau,\tau_1,\ldots,\tau_n, \tau'$ from ${\tau}$
to ${\tau'}$ in $\mathscr{T}$, then the \ALCQI signature contains role
names $Q_{\tau'}, Q_{\tau_i}$, for $i=1,\ldots,n$, and the following
role chain expression is generated by the mapping:
\[
	\pth{\tau}{\tau'}^\dag = Q_{\tau_1}\chain\ldots\chain Q_{\tau_n}\chain Q_{\tau'}, 
\]
In particular, the mapping uses the following notation: the inverse
role chain $(R_1\chain\ldots\chain R_n)^-$, for $R_i$ a role name,
stands for the chain $R_n^-\chain\ldots\chain R_1^-$, with $R_i^-$ an
inverse role, the expression
$\exists^{\lesseqgtr 1} R_1\chain\ldots\chain R_n\per C$ stands for
the \ALCQI concept expression
$\card{1}{R_1\per\ldots\per\exists^{\lesseqgtr 1} R_n}{C}$ and
$\forall R_1\chain\ldots\chain R_n\per C$ for the \ALCQI concept
expression $\forall R_1\per\ldots\per\forall R_n\per {C}$.  Thus,
since \DLRpm restricts to $q=1$ the cardinalities on any path of
length strictly greater than $1$ (see condition $2$ in
Def.~\ref{def:dlrpm}), the above notation shows that we remain within
the \ALCQI syntax when the mapping applies to cardinalities. If, e.g.,
we need to map the \DLRpm cardinality constraint $\EXISTR{q}{U_i}
R$ 
with $q>1$, then, to stay within the \ALCQI syntax,
$U_i$ 
must not be mentioned in any other projection in such a way that
$|\pth{\tau(R)}{\{U_i\}}|=1$. 
Finally, notice that the mapping introduces a concept name
$A_{R\!N}^{\tau_i}$ for each projected signature $\tau_i$ in the
projection signature graph dominated by $\tau(R\!N)$, i.e.,
$\tau_i\in\mathscr{T}_{\tau(R\!N)}$, to capture global reifications of
the various projections of $R\!N$ in the given KB. We also use the shortcut
$A_{R\!N}$ which stands for $A_{R\!N}^{\tau(R\!N)}$.

%
%

Intuitively, each node in the projection signature graph associated to
a \DLRpm KB denotes a relation projection and the mapping reifies each
of these projections. The target \ALCQI signature resulting from
mapping the \DLRpm KB of Example~\ref{exa:basic} is partially
presented in Fig.~\ref{fig:mapping}, together with the projection
signature graph (showed in Fig.~\ref{fig:multitree}).  Each node of
the graph is labelled with the corresponding global reification
concept ($A_{R_i}^{\tau_j}$), for each 
$R_i\in\mathcal{R}$ and each projected signature $\tau_j$ in the
projection signature graph dominated by $\tau(R_i)$, while the edges
are labelled by the roles ($Q_{\tau_i}$) needed for the reification.
\begin{figure}[t]
\centering
\resizebox{\columnwidth}{!}{


\begin{tikzpicture}[thick,shorten >=1pt, shorten <=1pt,->]
  \matrix (m) [matrix of math nodes, row sep=2.9em, column sep=0.3em]{
	 A_{R_1} & & A_{R_2} & \\
	A_{R_1}^{\{W_4\}} & A_{R_1}^{\{W_1,W_2,W_3\}},A_{R_2}^{\{W_1,W_2,W_3\}} & A_{R_2}^{\{V_1\}} & A_{R_2}^{\{V_2\}} \\
	& A_{R_1}^{\{W_1,W_2\}},A_{R_2}^{\{W_1,W_2\}} & A_{R_1}^{\{W_3\}},A_{R_2}^{\{W_3\}}\\
	A_{R_1}^{\{W_1\}},A_{R_2}^{\{W_1\}}  & A_{R_1}^{\{W_2\}},A_{R_2}^{\{W_2\}} \\
  };
  \path 
  	(m-1-1) edge node[left] {\scalebox{0.8}{$Q_{\{W_4\}}$}} (m-2-1)
  	(m-1-1) edge node[above right,sloped,anchor=south] {\scalebox{0.8}{$Q_{\{W_1,W_2,W_3\}}$}} (m-2-2)
  	(m-1-3) edge node[above right,sloped,anchor=south] {\scalebox{0.8}{$Q_{\{W_1,W_2,W_3\}}$}} (m-2-2)
  	(m-1-3) edge node[left] {\scalebox{0.8}{$Q_{\{V_1\}}$}} (m-2-3)
  	(m-1-3) edge node[right] {\scalebox{0.8}{$Q_{\{V_2\}}$}} (m-2-4)
  	(m-2-2) edge node[left] {\scalebox{0.8}{$Q_{\{W_1,W_2\}}$}} (m-3-2)
  	(m-2-2) edge node[right,sloped,anchor=south] {\scalebox{0.8}{$Q_{\{W_3\}}$}} (m-3-3)
  	(m-3-2) edge node[left,sloped,anchor=south] {\scalebox{0.8}{$Q_{\{W_1\}}$}} (m-4-1)
  	(m-3-2) edge node[right] {\scalebox{0.8}{$Q_{\{W_2\}}$}} (m-4-2)
  ;
\end{tikzpicture}
}
\vspace*{-4.5ex}
\caption{\label{fig:mapping} The \ALCQI signature generated by \Texa.}
\end{figure} 


To better clarify the need for the path function in the mapping,
notice that each \DLRpm relation is reified according to the
decomposition dictated by the projection signature graph it
dominates. Thus, to access, e.g., an attribute $U_j$ of a \DLRpm
relation ${R_i}$ it is necessary to follow the path through the
projections that use that attribute. Such a path, from the node
denoting the whole signature of the relation, ${\tau(R_i)}$, to the
node denoting the attribute $U_j$ is returned by the
$\pth{\tau(R_i)}{U_j}$ function.
For instance, considering the example from Figure~\ref{fig:mapping}, to
access the attribute $W_1$ of the relation $R_2$ in the expression
$(\selects{W_1}{C}{R_2})$, the mapping of the path
$\pth{\tau(R_2)}{\{W_1\}}^\dag$ is equal to the role chain
$Q_{\{W_1,W_2,W_3\}}\chain Q_{\{W_1,W_2\}}\chain Q_{\{W_1\}}$. This means that
$ (\selects{W_1}{C}{R_2})^\dag ~=~ A_{R_2} \sqcap \forall
Q_{\{W_1,W_2,W_3\}}\per\forall Q_{\{W_1,W_2\}}\per\forall Q_{\{W_1\}}\per C.$
Similar considerations can be done when mapping cardinalities over
relation projections.

We now present in details the mapping of a \DLRpm KB into a KB in
\ALCQI. Let $\KB = (\Tmc, \A)$ be a \DLRpm KB with signature
$(\mathcal{C},\mathcal{R},\Ob,\mathcal{U},\tau)$. The mapping
$\gamma(\KB)$ is assumed to be unsatisfiable (i.e., it contains the
axiom $\top\sqsubseteq \bot$) if the ABox contains the relation
assertion $R\!N(t)$ with $\tau(R\!N)\neq \tau(t)$, for some relation
$R\!N\in\Rmc$ and some tuple $t$. Otherwise,
$\gamma(\KB) = (\gamma(\Tmc),\gamma(\A))$ defines an \ALCQI KB as
follows:
\begin{align*}
  \gamma(\Tmc)  = {} &  \gamma_{\textit{dsj}} ~\cup %
  \bigcup_{R\!N\in\Rmc}\gamma_{\textit{rel}}(R\!N) ~\cup %
  \bigcup_{R\!N\in\Rmc}\gamma_{\textit{lobj}}({R\!N})
  ~\cup \\
  & \bigcup_{C_1\sqsubseteq C_2\in\KB}{C_1^\dag\sqsubseteq C_2^\dag}
  ~\cup%
 \bigcup_{R_1\sqsubseteq R_2\in\KB}{R_1^\dag\sqsubseteq
    R_2^\dag}\\
\gamma_\textit{dsj} = {} 
&~\bigl\{A_{R\!N_1}^{\tau_i}\sqsubseteq\neg A_{R\!N_2}^{\tau_j} \mid R\!N_1, R\!N_2\in\Rmc, \\
&\qquad \tau_i\in\mathscr{T}_{\tau(R\!N_1)}, \tau_j\in \mathscr{T}_{\tau(R\!N_2)}, |\tau_i|\geq 2, |\tau_j|\geq 2, \tau_i\neq \tau_j
  \bigr\}\\
%
\gamma_{\textit{rel}}(R\!N) =&~\bigcup_{\mathclap{\tau_i\in\mathscr{T}_{\tau(R\!N)}}}\quad
     \bigcup_{\mathrlap{\chd{\tau_i}{\tau_j}}}\quad\bigl\{
     A^{\tau_i}_{R\!N} \sqsubseteq \exists Q_{\tau_j}\per A^{\tau_j}_{R\!N},~\exists^{\geq 2}
     Q_{\tau_j}\per\top\sqsubseteq \bot
     \bigr\} \\
%
\gamma_{\textit{lobj}}({R\!N}) =&~\{
                                  \parbox[t]{0.5\textwidth}{$
A_{R\!N}\sqsubseteq\exists Q_{R\!N}\per A_{R\!N}^l,~
\exists^{\geq 2}Q_{R\!N}\per \top\sqsubseteq \bot,\\
A_{R\!N}^l \sqsubseteq \exists Q_{R\!N}^-\per A_{R\!N},~
\exists^{\geq 2} Q_{R\!N}^-\per \top\sqsubseteq \bot\}.$}
\end{align*}
Intuitively, $\gamma_\textit{dsj}$ ensures that relations with
different signatures are disjoint, thus, e.g., enforcing the union
compatibility. The axioms in $\gamma_{\textit{rel}}$ introduce
classical reification axioms for each relation and its relevant
projections. The axioms in $\gamma_{\textit{lobj}}$ make sure that
each local objectification differs from the global one while each role
$Q_{R\!N}$ defines a bijection.

To translate the ABox, we first map each individual $o\in\Ob$ in the \DLRpm
ABox \A to an \ALCQI individual $o$. 
Each relation instance occurring in \A 
is mapped via an injective function $\xi$ to a distinct
individual. That is, $\xi: T_\Ob(\mathcal{U})\to \mathcal{O}_\ALCQI$,
with $\mathcal{O}_\ALCQI = \Ob\cup \Ob^t$ being the set of individual names
in $\gamma(\KB)$, $\Ob \cap \Ob^t = \emptyset$ and
\vspace{-1ex}
$$\xi(t)  ~=~ 
                  \begin{cases}
                    o\in\Ob, &\text{if } t = \langle U\!:\!o\rangle \\
                    o\in\Ob^t, &\text{otherwise.}\\
                  \end{cases}
$$
Following~\cite{HorrocksSTT00}, the mapping $\gamma(\A)$ in
Fig.~\ref{fig:gammaA} introduces a new concept name $Q_o$ for each
individual $o \in\Ob$ and a new concept name $Q_t$ for each relation
instance $t$ occurring in \A, with each $Q_t$ restricted as follows:

\noindent
\begin{multline}
\label{eq:uni-tuple}
Q_t \sqsubseteq \exists^{\leq 1}\!\big(\pth{\tau(t)}{\{U_1\}}^\dag
\big)^-\per \\
\exists\!\big(\pth{\tau(t)}{\{U_2\}}^\dag
\big)\per\!Q_{o_2}\sqcap\!\ldots\!\sqcap
\exists\!\big(\pth{\tau(t)}{\{U_n\}}^\dag \big)\per\!Q_{o_n}
\end{multline}
%
Intuitively,~\eqref{eq:rei1} and~\eqref{eq:rei2} reify each relation
instance occurring in \A using the projection signature of the
relation instance itself. The
formulas~\eqref{eq:unique1}-\eqref{eq:unique2} together with the
axioms for concepts $Q_t$  guarantee that there is exactly one \ALCQI
individual reifying a given relation instance.
Clearly, the size of $\gamma(\KB)$ is polynomial in the size of $\KB$
under the same coding of the numerical parameters. 

\begin{figure}[t]
\begin{align}
\gamma(\A) =
& \label{eq:ob1}
\{ C\!N^\dag(o) \mid C\!N(o)\in\A \} ~\cup\\
& \label{eq:ob2}
\{ o_1\neq o_2 \mid o_1\neq o_2\in\A \} ~\cup~\{ o_1=o_2 \mid o_1= o_2\in\A \}~\cup\\
& \label{eq:rei1}
\{A^{\tau_i}_{R\!N}(\xi(t[\tau_i])) \mid R\!N(t)\in\A \textit{ and } \tau_i
  \in\mathscr{T}_{\tau(R\!N)}\} ~\cup\\
& \label{eq:rei2}
\{ Q_{\tau_j}\big(\xi(t[\tau_i]),\xi(t[\tau_j])\big) \mid R\!N(t)\in\A,
\tau_i\in\mathscr{T}_{\tau(R\!N)} \textit{ and } 
  \chd{\tau_i}{\tau_j} \}~\cup\\
& \label{eq:unique1}
\{Q_o(o)\mid o\in\Ob\}~\cup\\
& \label{eq:unique2}
\{Q_t(o_1)\mid t = \langle
  U_1\!:\!o_1,\ldots, U_n\!:\!o_n\rangle \text{ occurs in } \A\}.
\end{align}
\caption{The mapping $\gamma(\A)$}
\label{fig:gammaA}
\end{figure}

We are now able to state our main results.
\begin{theorem}
  \label{th:sat} A \DLRpm knowledge base $\KB$ is satisfiable iff the
  \ALCQI knowledge base $\gamma(\KB)$ is satisfiable.
\end{theorem}
\begin{proof}
  We assume that the \KB is consistently rewritten by substituting
  each attribute with its canonical representative, thus, we do not
  have to deal with the renaming of attributes. Furthermore, we extend
  the function $\imath$ to
  singleton tuples with the meaning that $\imath(\langle
  U_i:d_i\rangle)=d_i$.\\
  ($\Rightarrow$) Let
  $\Imc = (\Delta^\Imc, \cdot^\Imc, \rho, \imath,
  \ell_{R\!N_1},\ldots)$ be a model for a \DLRpm knowledge base
  $\KB$. To construct a model $\Jmc=(\Delta^\Jmc,\cdot^\Jmc)$ for the
  \ALCQI knowledge base $\gamma(\KB)$ we set
  $\Delta^\Jmc = \Delta^\Imc$, $o^\Jmc = o^\Imc$ for all $o\in\Ob$ and
  \begin{align}
    \label{int:xi}    
    [\xi(\langle U_1\!:\!o_1,\ldots,U_n\!:\!o_n\rangle)]^\Jmc =
    \imath(\langle
    U_1\!:\!o^\Imc_1,\ldots,U_n\!:\!o^\Imc_n\rangle).
  \end{align}
  Furthermore, we set:
  $(C\!N^\dag)^\Jmc=(C\!N)^\Imc$, for every atomic concept
  $C\!N\in\Cmc$, while for every ${R\!N}\in\Rmc$ and
  ${\tau_i\in\mathscr{T}_{\tau(R\!N)}}$ we set
  \begin{multline}\label{eq:R}
    (A_{R\!N}^{\tau_i})^\Jmc = \{\imath(\langle U_1:d_1,\ldots,U_k:d_k\rangle) \mid 
    \{U_1,\ldots,U_k\}=\tau_i \text{ and } \\ \exists t\in R\!N^\Imc\per t[U_1]=d_1,\ldots,t[U_k]=d_k\}.
  \end{multline}
  For each role name $Q_{\tau_i}$, $\tau_i\in\mathscr{T}$, we set
  \begin{multline}\label{eq:Q}
    (Q_{\tau_i})^\Jmc = \{(d_1,d_2)\in \Delta^\Jmc\times
    \Delta^\Jmc\mid \exists t\in {R\!N}^\Imc \text{ s.t. } d_1=\imath(t[{\tau_j}]), d_2=\imath(t[{\tau_i}])\\
    \text{ and } \chd{\tau_j}{\tau_i}, \text{ for some } {R\!N}\in\R\}.
  \end{multline}
  For every ${R\!N}\in\Rmc$ we set
  \begin{multline}\label{eq:loc}
  Q_{R\!N}^\Jmc = \{(d_1,d_2)\in \Delta^\Jmc\times \Delta^\Jmc\mid
    \exists t\in\Int{R\!N} \text{ s.t. } d_1=\imath(t) \text{ and } d_2=\ell_{R\!N}(t)\},
  \end{multline}
and
  \begin{align}\label{eq:Aloc}
  (A^l_{R\!N})^\Jmc = \{\ell_{R\!N}(t)\mid t\in \Int{R\!N}\}.
  \end{align}
We first show that $\Jmc$ is indeed a model of $\gamma(\Tmc)$.
\begin{enumerate}
\item $\Jmc\models \gamma_\textit{dsj}$. This is a direct consequence
  of the fact that $\imath$ is an injective function and that tuples
  with different signatures are different tuples.
\item $\Jmc\models \gamma_{\textit{rel}}(R\!N)$, for every
  ${R\!N\in\Rmc}$. We show that, for each $\tau_i,\tau_j$
  such that $\chd{\tau_i}{\tau_j}$ and $\tau_i\in\mathscr{T}_{\tau(R\!N)}$, it holds that
  $\Jmc\models A^{\tau_i}_{R\!N} \sqsubseteq \exists Q_{\tau_j}\per
  A^{\tau_j}_{R\!N}$
  and
  $\Jmc\models~\exists^{\geq 2} Q_{\tau_j}\per\top\sqsubseteq \bot$:
  \begin{itemize}
  \item
    $\Jmc\models A^{\tau_i}_{R\!N} \sqsubseteq \exists Q_{\tau_j}\per
    A^{\tau_j}_{R\!N}$.
    Let $d\in(A^{\tau_i}_{R\!N})^\Jmc$, by~(\ref{eq:R}),
    $\exists t\in {R\!N}^\Imc$ s.t. $d=\imath(t[{\tau_i}])$. Since
    $\chd{\tau_i}{\tau_j}$, then $\exists d'=\imath(t[{\tau_j}])$ and,
    by~(\ref{eq:Q}), $(d,d')\in Q_{\tau_j}^\Jmc$, while
    by~(\ref{eq:R}), $d'\in (A^{\tau_j}_{R\!N})^\Jmc$. Thus,
    $d\in (\exists Q_{\tau_j}\per A^{\tau_j}_{R\!N})^\Jmc$.
  \item
    $\Jmc\models~\exists^{\geq 2} Q_{\tau_j}\per\top\sqsubseteq \bot$.
    The fact that each $Q_{\tau_j}$ is interpreted as a funcional role is
    a direct consequence of the construction~(\ref{eq:Q}) and the fact
    that $\imath$ is an injective function.
  \end{itemize}
\item $\Jmc\models \gamma_{\textit{lobj}}(R\!N)$, for every
  ${R\!N\in\Rmc}$. Similar as above, considering the fact that each $\ell_{R\!N}$ is
  an injective function and equations~(\ref{eq:loc})-(\ref{eq:Aloc}).
\item $\Jmc\models {C_1^\dag\sqsubseteq C_2^\dag}$ and
  $\Jmc\models {R_1^\dag\sqsubseteq R_2^\dag}$. Since
  $\Imc\models {C_1\sqsubseteq C_2}$ and
  $\Imc\models R_1\sqsubseteq R_2$, it is enough to show the
  following:
  \begin{itemize}
  \item $d\in \Int{C} \text{ iff } d\in (C^\dag)^\Jmc$, for all \DLRpm\ concepts;
  \item $t\in \Int{R} \text{ iff } \imath(t)\in (R^\dag)^\Jmc$, for all \DLRpm\ relations.
  \end{itemize}
  Before we proceed with the proof, it is easy to show by structural
  induction that the following property holds:
  \begin{align}\label{prop:RN1}
    \text{If } \imath(t)\in R^{\dag\Jmc} \text{ then } \exists
    \imath(t')\in R\!N^{\dag\Jmc} 
    \text{ s.t. } t=t'[\tau(R)], \text{ for some } R\!N\in\R.
  \end{align}
  We now proceed with the proof by structural induction. The base
  cases, for atomic concepts and roles, are immediate form the
  definition of both ${C\!N}^\Jmc$ and ${R\!N}^\Jmc$.
  The cases where complex concepts and relations are constructed using
  either boolean operators, relation difference or global reification
  are easy to show. We thus show only the following cases.

  Let $d\in(\lreif R\!N)^\Imc$. Then, $d=\ell_{R\!N}(t)$ with
  $t\in R\!N^\Imc$. By induction, $\imath(t)\in A_{R\!N}^\Jmc$ and, by
  $\gamma_{\textit{lobj}}({R\!N})$, there is a $d'\in\Delta^\Jmc$
  s.t. $(\imath(t),d')\in Q_{R\!N}^\Jmc$ and
  $d'\in (A_{R\!N}^l)^\Jmc$. By~(\ref{eq:loc}), $d'=\ell_{R\!N}(t)$
  and, since $\ell_{R\!N}$ is injective, $d'=d$. Thus,  $d\in(\lreif
  R\!N)^{\dag\Jmc}$.

  Let $d\in(\exists^{\geq q}[U_i] R)^\Imc$. Then, there are different
  $t_1,\ldots,t_q\in R^\Imc$ s.t. $t_l[U_i]=d$, for all
  $l=1,\ldots,q$. By induction, $\imath(t_l)\in R^{\dag \Jmc}$ while,
  by~(\ref{prop:RN1}), $\imath(t'_l)\in R\!N^{\dag\Jmc}$, for
  some atomic relation ${R\!N}\in\R$ and a tuple $t'_l$
  s.t. $t_l=t'_l[\tau(R)]$. By~$\gamma_{\textit{rel}}({R\!N})$
  and~(\ref{eq:Q}), $(\imath(t'_l),\imath(t_l))\in
  (\pth{\tau({R\!N})}{\tau(R)}^\dag)^\Jmc$ and  $(\imath(t_l),d)\in
  (\pth{\tau({R})}{\{U_i\}}^\dag)^\Jmc$. Since $\imath$ is injective,
  $\imath(t_l)\neq \imath(t_j)$ when $l\neq j$,
  thus, $d\in(\exists^{\geq q}[U_i] R)^{\dag\Jmc}$.

  Let $t\in(\selects{U_i}{C}{R})^\Imc$. Then, $t\in R^\Imc$ and
  $t[U_i]\in C^\Imc$ and, by induction, $\imath(t)\in R^{\dag \Jmc}$
  and $t[U_i]\in C^{\dag \Jmc}$. As before, by
  $\gamma_{\textit{rel}}(R\!N)$ and by~(\ref{eq:Q}) and (\ref{prop:RN1}), we have 
  $(\imath(t),t[U_i])\in(\pth{\tau(R)}{\{U_i\}}^\dag)^\Jmc$.
   Since $\pth{\tau(R)}{U_i}^\dag$ is functional, then we have that
  $\imath(t)\in (\selects{U_i}{C}{R})^{\dag \Jmc}$.

  Let $t\in(\exists[U_1,\ldots,U_k] R)^{\Imc}$. Then, there is a tuple
  $t'\in R^\Imc$ s.t. $t'[U_1,\ldots,U_k]=t$ and, by induction,
  $\imath(t')\in R^{\dag\Jmc}$. As before, by
  $\gamma_{\textit{rel}}(R\!N)$ and by~(\ref{eq:Q}) and
  (\ref{prop:RN1}), we can show that
  $(\imath(t'),\imath(t))\in
  \pth{\tau(R)}{\{U_1,\ldots,U_k\}}^{\dag\Jmc}$ and thus it follows that 
  $\imath(t)\in(\exists[U_1,\ldots,U_k] R)^{\dag\Jmc}$.

  All the other cases can be proved in a similar way. We now show the
  vice versa.

  \smallskip

  Let $d\in(\lreif R\!N)^{\dag\Jmc}$. Then, $d\in (A_{R\!N}^l)^\Jmc$
  and $d=l_{R\!N}(t)$, for some $t\in
  {R\!N}^\Imc$, i.e., $d\in(\lreif R\!N)^\Imc$.

  Let $d\in(\exists^{\geq q}[{U_i}] R)^{\dag \Jmc}$. Then, there are
  different  $d_1,\ldots,d_q\in\Delta^\Jmc$ such that
  $(d_l,d)\in (\pth{\tau(R)}{\{U_i\}}^\dag)^{\Jmc}$ and
  $d_l\in R^{\dag \Jmc}$, for $l=1,\ldots,q$. By induction, each
  $d_l=\imath(t_l)$ and $t_l\in R^\Imc$. Since $\imath$ is injective,
  then $t_l\neq t_j$ for all $l,j=1,\ldots,q$, $l\neq j$. We need to
  show that $t_l[U_i] = d$, for all $l=1,\ldots,q$. By~(\ref{eq:Q})
  and the fact that $(d_l,d)\in (\pth{\tau(R)}{\{U_i\}}^\dag)^{\Jmc}$,
  then $d=\imath(t_l[U_i])=t_l[U_i]$.

  Let $\imath(t)\in(\selects{U_i}{C}{R})^{\dag\Jmc}$.  Then,
  $\imath(t)\in R^{\dag \Jmc}$ and, by induction, $t\in R^\Imc$. Let
  $t[U_i]=d$. We need to show that $d\in C^\Imc$. By
  $\gamma_{\textit{rel}}(R\!N)$ and by~(\ref{eq:Q}) and
  (\ref{prop:RN1}), it follows that
  $(\imath(t),d)\in (\pth{\tau(R)}{\{U_i\}}^\dag)^\Jmc$, then
  $d\in C ^{\dag\Jmc}$ and, by induction, $d\in C ^{\Imc}$.

  Let $\imath(t)\in(\exists[U_1,\ldots,U_k] R)^{\dag\Jmc}$. Then, there is
  $d\in\Delta^\Jmc$ s.t.
  $$(d,\imath(t)) \in (\pth{\tau(R)}{\{U_1,\ldots,U_k\}}^\dag)^\Jmc$$ 
  and $d\in R^{\dag\Jmc}$.  By induction, $d=\imath(t')$ and
  $t'\in R^\Imc$. By the definionition of the mapping of paths
  and~(\ref{eq:Q}), $\imath(t)= \imath(t'[U_1,\ldots,U_k])$, i.e.,
  $t=t'[U_1,\ldots,U_k]$. Thus,
  $t\in(\exists[U_1,\ldots,U_k] R)^{\Imc}$.

  \smallskip
  We now show that $\Jmc$ is a model of $\gamma(\A)$.
  
  Concerning axioms in~\eqref{eq:ob1} and~\eqref{eq:ob2} they are
  satified by construction. $\Jmc$ also satisfies axioms
  in~\eqref{eq:rei1} and in~\eqref{eq:rei2} due to~\eqref{eq:R}
  and~\eqref{eq:Q}, respectively, and the interpretation of $\xi$
  in~\eqref{int:xi}. Concerning axioms
  in~\eqref{eq:unique1}-\eqref{eq:unique2}, we set
  $Q_o^\Jmc = \{o^\Imc\}$, for each $o\in\Ob$, and
  $Q_t^\Jmc = \{o_1^\Imc\}$, for each tuple
  $t = \langle U_1\!:o_1,\ldots,U_n\!:\!o_n\rangle$ occurring in
  $\A$. We finally show that $\Jmc$ satisfies axiom~\eqref{eq:uni-tuple}
  by considering, w.l.o.g., the case of binary tuples,
  $t = \langle U_1\!:\!o_1, U_2\!:\!o_2\rangle$. Then,
  $\pth{\tau(t)}{\{U_1\}}^\dag = Q_{U_1}$ and
  $\pth{\tau(t)}{\{U_2\}}^\dag = Q_{U_2}$. Assume that
  $o_1^\Jmc\in Q_t^\Jmc$ and that there are objects
  $d_1, d_2, d_3, d_4\in \Delta^\Jmc$ such that
  $(d_1,o_1^\Jmc), (d_2,o_1^\Jmc)\in Q_{U_1}^\Jmc$,
  $(d_1,d_3), (d_2,d_4)\in Q_{U_2}^\Jmc$ and
  $d_3, d_4\in Q_{o_2}^\Jmc$. We need to show that $d_1 = d_2$. We
  first notice that, since concepts $Q_o$ are interpreted as
  singleton, $d_3 = d_4 = o_2^\Jmc$. Furthermore, by~\eqref{eq:Q},
  $d_1 = \imath(t_1)$ and $d_2 = \imath(t_2)$, with $t_1 = \langle
  U_1\!:o_1^\Jmc, U_2\!:\!d_3\rangle$ and $t_2 = \langle
  U_1\!:o_1^\Jmc, U_2\!:\!d_4\rangle$ and thus $t_1 = t_2$. Since
  $\imath$ is injective, then $d_1 = d_2$.
\end{enumerate}

\smallskip

\noindent
($\Leftarrow$) Let $\Jmc=(\Delta^\Jmc,\cdot^\Jmc)$ be a model for the
knowledge base $\gamma(\KB)$.  Without loss of generality, we can
assume that $\Jmc$ is a \textit{forest model}. We then construct a model
$\Imc = (\Delta^\Imc, \cdot^\Imc, \rho, \imath, \ell_{R\!N_1},\ldots)$
for a \DLRpm knowledge base $\KB$. We set:
$\Delta^\Imc = \Delta^\Jmc$, $o^\Imc = o^\Jmc$ for all $o\in\Ob$, $C\!N^\Imc = (C\!N^\dag)^\Jmc$, for every
atomic concept $C\!N\in\Cmc$, while, for every ${R\!N}\in\Rmc$, we set:
\begin{multline}\label{eq:RN}
  R\!N^\Imc = \{t=\langle U_1\!:\!d_1,\ldots,U_n\!:\!d_n\rangle \in T_{\Delta^\Imc}(\tau(R\!N))\mid
  \exists d\in A_{R\!N}^\Jmc \text{ s.t. }\\ (d,t[U_i])\in (\pth{\tau(R\!N)}{\{U_i\}}^\dag)^\Jmc \text{ for } i=1,\ldots,n\}.
\end{multline}
Notice that~\eqref{eq:RN} defines a bijection between objects in
$\ALCQI$ reifying tuples and tuples themselves. Indeed, since $\Jmc$
satisfies $\gamma_{\textit{rel}}(R\!N)$, for every
$d\in A_{R\!N}^\Jmc$ there is a unique tuple
$\langle U_1\!:\!d_1,\ldots,U_n\!:\!d_n\rangle \in R\!N^\Imc$---thus we say
that $d$ \emph{generates}
$\langle U_1\!:\!d_1,\ldots,U_n\!:\!d_n\rangle$ and, in symbols,
$d\to \langle U_1\!:\!d_1,\ldots,U_n\!:\!d_n\rangle$. Furthermore,
since \Jmc is forest shaped, to each tuple whose components are not in
the ABox corresponds a unique $d$ that generates it. On the other
hands, since $\Jmc$ satisfies axiom~\eqref{eq:uni-tuple}, then also
for tuples occurring in the ABox there is a unique $d$ that generates
them.  Thus, let $d\to \langle U_1\!:\!d_1,\ldots,U_n\!:\!d_n\rangle$,
by setting $\imath(\langle U_1\!:\!d_1,\ldots,U_n\!:\!d_n\rangle) = d$
and
\begin{multline}\label{eq:iota}
  \imath(\langle U_1\!:\!d_1,\ldots,U_n\!:\!d_n\rangle[\tau_i])=d_{\tau_i}, \text{ s.t. }\\
  (d,d_{\tau_i})\in(\pth{\{U_1,\ldots,U_n\}}{\tau_i}^\dag)^\Jmc,
\end{multline}
for all ${\tau_i\in\mathscr{T}}$ s.t.
$\tau_i\subset \{U_1,\ldots,U_n\}$, then, the function $\imath$ is as required.

By setting
\begin{multline}\label{eq:lobj}
  \ell_{R\!N}(\langle U_1\!:\!d_1,\ldots,U_n\!:\!d_n\rangle) =
  d, \text{ s. t. }\\
  (\imath(\langle U_1\!:\!d_1,\ldots,U_n\!:\!d_n\rangle),d)\in
  Q_{R\!N}^\Jmc,
\end{multline}
then, by $\gamma_{\textit{lobj}}(R\!N)$, both $Q_{R\!N}$ and its inverse are
interpreted as a functional roles by \Jmc, thus
the function $\ell_{R\!N}$ is as required.

It is easy to show by structural induction that the following property holds:
\begin{align}\label{prop:RN}
\text{If } t\in R^\Imc \text{ then } \exists t'\in R\!N^\Imc \text{
  s.t. } 
t=t'[\tau(R)], \text{ for some } R\!N\in\R.
\end{align}
We now show that $\Imc$ is indeed a model of $\KB$. We first show that
$\Imc\models \Tmc$, i.e., 
$\Imc\models {C_1\sqsubseteq C_2}$ and
$\Imc\models R_1\sqsubseteq R_2$. As before, since
$\Jmc\models {C_1^\dag\sqsubseteq C_2^\dag}$ and
$\Jmc\models R_1^\dag\sqsubseteq R_2^\dag$, it is enough to show the
following:
  \begin{itemize}
  \item $d\in \Int{C} \text{ iff } d\in (C^\dag)^\Jmc$, for all \DLRpm\ concepts;
  \item $t\in \Int{R} \text{ iff } \imath(t)\in (R^\dag)^\Jmc$, for all \DLRpm\ relations.
  \end{itemize}
  The proof is by structural induction. The base cases are trivially
  true. Similarly for the boolean operators, difference between
  relations and global
  reification. We thus show only the following cases.

  Let $d\in(\lreif R\!N)^\Imc$. Then, $d=\ell_{R\!N}(t)$ with
  $t\in R\!N^\Imc$. By induction, $\imath(t)\in A_{R\!N}^\Jmc$ and, by
  $\gamma_{\textit{lobj}}({R\!N})$, there is a $d'\in\Delta^\Jmc$
  s.t. $(\imath(t),d')\in Q_{R\!N}^\Jmc$ and
  $d'\in (A_{R\!N}^l)^\Jmc$.  By~(\ref{eq:lobj}), $d=d'$ and
  thus, $d\in (\lreif R\!N)^{\dag\Jmc}$.

  Let $d\in(\exists^{\geq q}[{U_i}] R)^{\Imc}$. Then, $U_i\in\tau(R)$
  and there are different $t_1,\ldots,t_q\in R^\Imc$
  with $t_l[U_i]=d$, for all $l=1,\ldots,q$. For each $t_l$,
  by~(\ref{prop:RN}), there must exist some element
  $t'_l\in R\!N^\Imc \text{ such that } t_l=t'_l[\tau(R)]$, for some
  $R\!N\in\R$, while, by induction, $\imath(t_l)\in R^{\dag \Jmc}$ and
  $\imath(t'_l)\in R\!N^{\dag\Jmc}$. Thus, $t'_l[U_i]=t_l[U_i]=d$ and,
  by~(\ref{eq:RN}), it then follows that
  $(\imath(t'_l),d)\in (\pth{\tau(R\!N)}{\{U_i\}}^\dag)^\Jmc$ while,
  by~(\ref{eq:iota}), we have
  $(\imath(t'_l),\imath(t_l))\in
  (\pth{\tau(R\!N)}{\tau(R)})^{\dag\Jmc}$.  Since \DLRpm allows only
  for knowledge bases with a projection signature graph being a
  multitree, then,
  $$\pth{\tau(R\!N)}{\{U_i\}}^\dag =
  \pth{\tau(R\!N)}{\tau(R)}^\dag\chain \pth{\tau(R)}{\{U_i\}}^\dag.$$
  Thus, $(\imath(t_l),d)\in (\pth{\tau(R)}{\{U_i\}}^\dag)^\Jmc$ and,
  since $\imath$ is injective, then, $\imath(t_l)\neq \imath(t_j)$
  when $l\neq j$. Thus,  $d\in(\exists^{\geq q}[{U_i}] R)^{\dag\Jmc}$.

%
  Let $t\in(\selects{U_i}{C}{R})^\Imc$. Then, $t\in R^\Imc$, $U_i\in\tau(R)$ and
  $t[U_i]=d\in C^\Imc$. By induction, $\imath(t)\in R^{\dag\Jmc}$ and
  $d\in C^{\dag\Jmc}$. As before, by~(\ref{eq:RN}),~(\ref{eq:iota}) and~(\ref{prop:RN}),
  we can show that
  $(\imath(t),d)\in (\pth{\tau(R)}{\{U_i\}}^\dag)^\Jmc$ and, since
  $\pth{\tau(R)}{\{U_i\}}^{\dag}$ is functional, then
  $\imath(t)\in(\selects{U_i}{C}{R})^{\dag\Jmc}$.

  Let $t\in(\exists[U_1,\ldots,U_k] R)^{\Imc}$. Then, there is a tuple
  $t'\in R^\Imc$ s.t. $t'[U_1,\ldots,U_k]=t$ and, by induction,
  $\imath(t')\in R^{\dag\Jmc}$. As before,
  by~(\ref{eq:iota}) and (\ref{prop:RN}), we can show that
  $(\imath(t'),\imath(t))\in
  \pth{\tau(R)}{\{U_1,\ldots,U_k\}}^{\dag\Jmc}$ and thus
  $\imath(t)\in(\exists[U_1,\ldots,U_k] R)^{\dag\Jmc}$.

  All the other cases can be proved in a similar way. We now show the
  converse direction.

  \smallskip

  Let $d\in(\lreif R\!N)^{\dag\Jmc}$. Then, $d\in (A_{R\!N}^l)^\Jmc$
  and, by $\gamma_{\textit{lobj}}({R\!N})$, there is a
  $d'\in\Delta^\Jmc$ s.t. $(d',d)\in Q_{R\!N}^\Jmc$ and
  $d'\in A_{R\!N}^\Jmc$. By induction, $d'=\imath(t')$ with
  $t'\in {R\!N}^\Imc$ and thus, $(\imath(t'),d)\in Q_{R\!N}^\Jmc$ and,
  by~(\ref{eq:lobj}), $\ell_{R\!N}(t') = d$, i.e., $d\in(\lreif R\!N)^\Imc$.

  Let $d\in(\exists^{\geq q}[{U_i}] R)^{\dag \Jmc}$.
  Then, ${U_i}\in\tau(R)$ and there are different $d_1,\ldots,d_q\in\Delta^\Jmc$ s.t.
  $(d_l,d)\in (\pth{\tau(R)}{\{U_i\}}^\dag)^{\Jmc}$ and
  $d_l\in R^{\dag \Jmc}$, for $l=1,\ldots,q$. By induction, each
  $d_l=\imath(t_l)$ and $t_l\in R^\Imc$. Since $\imath$ is injective,
  then $t_l\neq t_j$ for all $l,j=1,\ldots,q$, $l\neq j$. We need to
  show that $t_l[U_i] = d$, for all
  $l=1,\ldots,q$. By~(\ref{prop:RN}), there exists a
  $t'_l\in R\!N^\Imc \text{ such that } t_l=t'_l[\tau(R)], \text{ for some
  } R\!N\in\R$ and, by~(\ref{eq:iota}), it holds that
  $(\imath(t'_l),\imath(t_l))\in
  (\pth{\tau(R\!N)}{\tau(R)}^\dag)^\Jmc$.
  Since $(\imath(t_l) ,d) \in (\pth{\tau(R)}{\{U_i\}}^\dag)^\Jmc$ and
  $\textsc{path}_{\mathscr{T}}$ is functional in \DLRpm,
  then, $(\imath(t'_l) ,d) \in (\pth{\tau(R\!N)}{\{U_i\}}^\dag)^\Jmc$
  and, by~(\ref{eq:RN}), $t'_l[U_i]=t_l[U_i] =d$.

%
  Let $\imath(t)\in(\selects{U_i}{C}{R})^{\dag\Jmc}$. Then,
  $\imath(t)\in R^{\dag \Jmc}$ and, by induction, $t\in R^\Imc$. Let
  $t[U_i]=d$. We need to show that $d\in C^\Imc$. As before,
  by~(\ref{prop:RN}) and~(\ref{eq:iota}), we have that
  $(\imath(t),d)\in (\pth{\tau(R)}{\{U_i\}}^\dag)^\Jmc$. Then
  $d\in C ^{\dag\Jmc}$ and, by induction, $d\in C ^{\Imc}$.

  Let $\imath(t)\in(\exists[U_1,\ldots,U_k] R)^{\dag\Jmc}$. Then, there is
  $d\in\Delta^\Jmc$ s.t.
  \begin{equation}
    \label{eq:path-tuple}
     (d,\imath(t)) \in (\pth{\tau(R)}{\{U_1,\ldots,U_k\}}^\dag)^\Jmc
  \end{equation}
  and $d\in R^{\dag\Jmc}$.  By induction, $d=\imath(t')$ and
  $t'\in R^\Imc$. By~(\ref{prop:RN}), there is a tuple
  $t''\in R\!N^\Imc$ s.t. $t' = t''[\tau(R)]$ and, by~(\ref{eq:iota}),
  $(\imath(t''),\imath(t'))\in (\pth{\tau(R\!N)}{\tau(R)}^\dag)^\Jmc$
  and thus, by~\eqref{eq:path-tuple},
  $(\imath(t''),\imath(t))\in
  (\pth{\tau(R\!N)}{\{U_1,\ldots,U_k\}}^\dag)^\Jmc$ and thus
  $t = t''[\{U_1,\ldots,U_k\}]$. Since
  $\{U_1,\ldots,U_k\}\subseteq \tau(R) \subseteq \tau( R\!N)$, then,
  $t = t''[\{U_1,\ldots,U_k\}] = (t''[\tau(R)])[U_1,\ldots,U_k] = t'[U_1,\ldots,U_k]$, i.e.,
  $t\in(\exists[U_1,\ldots,U_k] R)^{\Imc}$.

To show that $\Imc\models\A$, notice that \Imc satisfies both concept
assertions and individual assertions by construction. We need to show
that \Imc satisfies also relation assertions. Let $R\!N(t)\in\A$, with
$t=\langle U_1\!:\!o_1,\ldots,U_n\!:\!o_n\rangle$,
then, since \Jmc satisfies $\gamma(\A)$, and in particular
axiom~\eqref{eq:rei1}, then there exists $d = \xi(t) \in
A_{R\!N}^\Jmc$. By~\eqref{eq:rei2}, $(d,o_i^\Jmc)\in
(\pth{\tau(R\!N)}{\{U_i\}}^\dag)^\Jmc$ and, by~\eqref{eq:RN},
$t^\Imc\in R\!N^\Imc$.
\hfill\qed

\end{proof}

As a direct consequence of the above theorem and the fact that \DLR is a sublanguage of \DLRpm, we have that

\begin{corollary}
  Reasoning in \DLRpm is an \ExpTime-complete problem.
\end{corollary}

\section{Implementation of a \DLRpm API}
\label{sec:api}

We have implemented the framework discussed in this paper. DLRtoOWL is a Java library fully implementing \DLRpm reasoning services. The library is based on the tool ANTLR4 to parse serialised input, and on OWLAPI4 for the OWL2 encoding. The system includes JFact, the Java version of the popular Fact++ reasoner. DLRtoOWL provides a Java \DLR API package to allow developers to create, manipulate, serialise, and reason with \DLRpm knowledge bases in their Java-based application, extending in a compatible way the standard OWL API with the \DLRpm \textsc{tell} and \textsc{ask} services.

During the development of this new library we strongly focused on performance. Since the OWL encoding is only possible if we have already built the \ALCQI projection signature multitree, in principle the program should perform two parsing rounds: one to create the multitree and the other one to generate the OWL mapping. We faced this issue using dynamic programming: during the first (and only) parsing round we store in a data structure each axiom that we want to translate in OWL and, after building the multitree, by the dynamic programming technique we build on-the-fly a Java class which generates the required axioms.

\section{Conclusions}
\label{sec:conc}

We have introduced the very expressive \DLRp{}\negmedspace description logic, which
extends \DLR 
with database oriented constraints. \DLRp is expressive enough to
cover directly and more thoroughly the EER, UML, and ORM conceptual
data models, among others.
Although reasoning in \DLRp is undecidable, we show that a simple syntactic constraint on KBs restores decidability. In fact, the resulting logic \DLRpm has the same complexity (\ExpTime-complete) as the basic \DLR language. In other words, handling database constraints does not increase the complexity of reasoning in the logic.
To enhance the use and adoption of \DLRpm{}\negmedspace, 
we have developed an API that fully implements reasoning for this language, and maps input knowledge bases into OWL. Using a standard OWL reasoner, we are able to provide a variety of \DLRpm reasoning services. 

We plan to investigate the problem of query answering under \DLRpm
ontologies 
and to check whether the complexity for this problem 
can be lifted from known results in \DLR to \DLRpm{}\negmedspace.


\bibliographystyle{splncs03}
\bibliography{biblio}

\end{document}